%% file: main.tex
\documentclass{article}
\usepackage[preprint]{neurips_2025}	
\usepackage{amsthm}
\usepackage{mathtools}
\usepackage{makecell, multirow}
\newtheorem{theorem}{Theorem}
\newcommand{\tabincell}[2]{\begin{tabular}{@{}#1@{}}#2\end{tabular}}

\usepackage{tikz}
\usetikzlibrary{
    arrows.meta,
    positioning,
    fit,
    calc
}

\usepackage{times}
\usepackage{soul}
\usepackage{url}
\usepackage[hidelinks]{hyperref}
\usepackage{graphicx}
\usepackage{amsmath}
\usepackage{amsthm}
\usepackage{booktabs}
\usepackage{algorithm}
\usepackage{algorithmic}
\usepackage{natbib} 
\usepackage{amsfonts}
\usepackage{latexsym}
\usepackage{subcaption}
\usepackage{multirow}

\title{Graph Neural Networks are Heuristics}
\author{%
  Yimeng Min \\
  Department of Computer Science\\
  Cornell University \\
  Ithaca, NY, USA \\
  \texttt{min@cs.cornell.edu} \\
  \And
  Carla P. Gomes\\
  Department of Computer Science\\
  Cornell University \\
  Ithaca, NY, USA \\
  \texttt{gomes@cs.cornell.edu}\\
}

\begin{document}
\maketitle

\begin{abstract}
Graph neural networks are usually treated as auxiliaries for combinatorial optimization: they imitate algorithms, guide search, or supply scores to classical procedures. 
We show that this auxiliary role is not intrinsic. 
A GNN can itself be a heuristic. 
For the Euclidean Travelling Salesman Problem, we train a non-autoregressive GNN with no labels, rewards, sequential decoding, search, or local improvement. 
A differentiable Hamiltonian-cycle objective is the only supervision. 
The trained model produces a complete tour in one forward pass, while dropout and snapshots from a single training trajectory provide solution diversity without engineered moves. 
The heuristic is therefore learned, not programmed. 
It is also fast: batched inference remains in the millisecond regime on GPUs. 
Experiments on TSP100, TSP200, and TSP500 show that the model consistently improves over nearest-neighbor greedy baselines. 
These results identify unsupervised GNNs as a class of fast learned heuristics for combinatorial optimization.
\end{abstract}

\input{Sections/intro}
\input{Sections/ml4CO}

\input{Sections/firstprinciple}
\input{Sections/background}

\input{Sections/model}
\input{Sections/contribution}
\input{Sections/methods}
\input{Sections/harmonics}

\input{Sections/dropout}

\input{Sections/snapshot}
\input{Sections/Training}

\input{Sections/results}

\input{Sections/conclusion}
\input{Sections/acknowledgement}

\section*{Ethical Statement}
There are no ethical issues.
\bibliography{utspreference}
\bibliographystyle{plainnat}    % common author–year style

\end{document}

%% file: Sections/intro.tex
\section{Introduction}
Heuristics are the grammar of problem-solving: the means by which intelligence—human or artificial—moves from paralysis in the face of combinatorial explosion to progress through approximation. At their core, heuristics are not mere tricks or shortcuts, but strategies for reasoning under constraint, where exact solutions are unattainable or impractical.    The Travelling Salesman Problem (TSP) has long served as a proving ground for such strategies. Exact solvers like Concorde~\citep{concorde} achieve optimality at exponential cost, while carefully engineered heuristics such as Lin-Kernighan-Helsgaun~\citep{lin1973effective,helsgaun2000effective} embody layers of local improvement rules that deliver state-of-the-art approximations. Simpler constructive approaches (nearest neighbor, farthest insertion) reveal the essential tradeoff between efficiency and accuracy that has always defined heuristic reasoning.  

From this broader perspective, the study of TSP is part of a long tradition: the search for principled ways to navigate intractable problems through structure, approximation, and design. This stance has been formalized into algorithmic procedures that exploit structure and domain knowledge to yield good-enough solutions. Heuristics, in this view, are systematic ways of trading optimality for efficiency while still capturing the essential structure of the problem~\citep{halim2019combinatorial,christofides1976worst,karlin2021slightly}.

We take a different view: heuristics need not be hand-crafted rules or search procedures. They can emerge directly from structure and learning. Building on the \emph{Structure As Search} framework~\citep{min2023unsupervised,min2023unsupervisedw,min2025structure}, we formulate TSP as direct permutation learning. Our non-autoregressive (NAR) method generates Hamiltonian cycles without supervision, explicit search, or sequential decoding. The key insight is that TSP's inherent structure---shortest Hamiltonian cycles---constrains the solution space sufficiently for neural networks to learn effective solutions. We employ Gumbel-Sinkhorn relaxation during training and Hungarian algorithm decoding at inference, enabling end-to-end optimization. In this sense, our method constitutes a new form of heuristic. It is not designed by hand, but arises from learning from instances. Like classical heuristics, it does not guarantee optimality; yet, like the best of them, it produces consistently strong solutions across instance sizes. Our results show that heuristics can now be \emph{learned} rather than \emph{engineered}, pointing toward a paradigm where structural constraints and generative models replace explicit search as the foundation of combinatorial optimization, our implementation can be found at \href{https://github.com/yimengmin/GNNs-are-Strong-Heuristics}{\textcolor{blue}{GNNs-are-Strong-Heuristics}}.

%% file: Sections/ml4CO.tex
\section{Background}
Combinatorial optimization problems lie at the core of operations research, theoretical computer science, and many real-world decision-making systems. 
Classical approaches range from exact solvers, which guarantee optimality at exponential cost, to heuristic algorithms that trade optimality for scalability.
Among these, greedy heuristics occupy a central role due to their simplicity, interpretability, and strong empirical performance on many structured instance families.

Graph Neural Networks (GNNs) have recently gained prominence as a tool for automating heuristic design in combinatorial optimization, moving beyond traditional manual approaches~\citep{khalil2017learning,bengio2021machine,cappart2023combinatorial}. Rather than hand-crafting decision rules, learning-based methods aim to infer solution strategies directly from data, leveraging structural regularities across problem instances.
This paradigm has led to a wide spectrum of approaches, including reinforcement learning with autoregressive decoding, supervised learning combined with post hoc search, and unsupervised formulations that embed combinatorial constraints into differentiable objectives.

%% file: Sections/firstprinciple.tex
\paragraph{What Does It Mean to Outperform Greedy?}
Recent work has argued that modern graph neural networks for combinatorial optimization do not reliably outperform classical greedy heuristics~\citep{angelini2023modern,boettcher2023inability,schuetz2023reply,schuetz2023reply2}.
This observation is important, but its interpretation is not automatic.
Greedy methods are designed to exploit local structure.
They are therefore strongest precisely on instances where local choices already determine high-quality global solutions.

We take the comparison with greedy as a diagnostic rather than as a verdict.
The relevant question is not whether learning defeats greedy in every favorable regime.
It is whether a learned model can represent constraints that are invisible to purely local decision rules.
For TSP, this distinction is sharp.
A tour is not a sequence of good nearest-neighbor moves; it is a Hamiltonian cycle, and its quality depends on globally consistent coordination among all vertices.

We therefore remove the usual auxiliary mechanisms: supervision, reinforcement rewards, autoregressive decoding, explicit search, and post hoc local improvement.
In this setting, any gain over greedy must come from the representation and the objective alone.
A persistent improvement over greedy is then evidence that the model has encoded non-local combinatorial structure, rather than merely exploiting local geometric cues.

Guided by this criterion, we use a non-autoregressive framework that predicts a complete permutation in one forward pass.
The permutation defines the Hamiltonian cycle directly, so optimization takes place over global tour structure rather than over incremental path construction.
By placing the Hamiltonian-cycle constraint inside the learning objective, the model is forced to coordinate all decisions jointly.
Thus, outperforming greedy in this setting has a precise meaning: the GNN functions as a learned heuristic whose advantage arises from global consistency, not from imitation, search, or engineered refinement.

%% file: Sections/background.tex
\section{Unsupervised Learning for TSP}
\subsection{Matrix Formulation of TSP}
The TSP seeks the shortest Hamiltonian cycle through $n$ cities with coordinates $x_i \in \mathbb{R}^{n \times 2}$, formulated as:
\begin{equation}
\min_{\sigma \in S_n} \sum_{i=1}^{n} d(x_{\sigma(i)}, x_{\sigma(i+1)}),
\end{equation}
where $d(x_i, x_j) = \|x_i - x_j\|_2$ and $\sigma(n+1) := \sigma(1)$. We encode Hamiltonian cycles via permutation matrices. The cyclic shift matrix $\mathbb{V} \in \{0,1\}^{n \times n}$ is defined as:
\begin{equation}
\mathbb{V}_{i,j} = \begin{cases}
1 & \text{if } j \equiv (i+1) \pmod{n} \\
0 & \text{otherwise}
\end{cases}.
\end{equation}
% yielding:
% \begin{equation}\label{eq:v}
% \mathbb{V} = \begin{pmatrix}
% 0 & 1 & 0 & 0 & \cdots & 0 & 0 \\
% 0 & 0 & 1 & 0 & \cdots & 0 & 0 \\
% 0 & 0 & 0 & 1 & \cdots & 0 & 0 \\
% \vdots & \vdots & \vdots & \vdots & \ddots & \vdots & \vdots \\
% 0 & 0 & 0 & 0 & \cdots & 1 & 0 \\
% 0 & 0 & 0 & 0 & \cdots & 0 & 1 \\
% 1 & 0 & 0 & 0 & \cdots & 0 & 0
% \end{pmatrix}.
% \end{equation}
Matrix $\mathbb{V}$ represents the canonical cycle $1 \rightarrow 2 \rightarrow \cdots \rightarrow n \rightarrow 1$. Any Hamiltonian cycle matrix $\mathcal{H} \in \mathbb{R}^{n \times n}$ is generated via similarity transformation: $\mathcal{H} = \mathbf{P} \mathbb{V} \mathbf{P}^\top$ for permutation matrix $\mathbf{P} \in S_n$~\citep{min2023unsupervisedw}. Given distance matrix $\mathbf{D} \in \mathbb{R}^{n \times n}$, the TSP objective becomes:
\begin{equation}
\min_{\mathbf{P} \in S_n} \langle \mathbf{D}, \mathbf{P} \mathbb{V} \mathbf{P}^\top \rangle,
\end{equation}
where $\langle \mathbf{A}, \mathbf{B} \rangle = \text{tr}(\mathbf{A}^\top \mathbf{B})$. To enable end-to-end gradient-based optimization, we relax discrete $\mathbf{P}$ to soft permutation $\mathbb{T} \in \mathbb{R}^{n \times n}$, minimizing:
\begin{equation}\label{eq:tsploss}
\mathcal{L}_{\text{TSP}} = \langle \mathbf{D},\mathbb{T} \mathbb{V} \mathbb{T}^\top \rangle.
\end{equation}

%% file: Sections/model.tex
\subsection[Soft-to-Hard Permutation Extraction]{From Soft Permutation $\mathbb{T}$ to Hard Permutation $\mathbf{P}$}

We build the soft permutation $\mathbb{T}$ following \citep{min2023unsupervisedw,min2025structure}. The GNN processes geometric features $f_0 \in \mathbb{R}^{n \times 2}$ (city coordinates) and adjacency matrix:
\begin{equation}
A = e^{-\mathbf{D}/s},
\end{equation}
where $s$ scales the distance matrix $\mathbf{D}$. Network output generates scaled logits:
\begin{equation}\label{eq:fgnn}
\mathcal{F} = \alpha \tanh(f_{\mathrm{GNN}}(f_0, A)),
\end{equation}
where $f_{\mathrm{GNN}} : \mathbb{R}^{n \times 2} \times \mathbb{R}^{n \times n} \to \mathbb{R}^{n \times n}$ and $\alpha$ controls scaling.

Differentiable permutation approximation via Gumbel-Sinkhorn:
\begin{equation}
\mathbb{T} = \mathrm{GS}\left(\frac{\mathcal{F} + \gamma \epsilon}{\tau},\, l\right),
\end{equation}
where $\epsilon$ denotes i.i.d. Gumbel noise, $\gamma$ sets noise magnitude, $\tau$ controls relaxation temperature, and $l$ specifies Sinkhorn iterations~\citep{mena2018learning}.

Discrete permutation extraction at inference:
\begin{equation}
\mathbf{P} = \mathrm{Hungarian}\left(-\frac{\mathcal{F} + \gamma \epsilon}{\tau}\right).
\label{eq:hung}
\end{equation}

Final tour reconstruction: $\mathbf{P} \mathbb{V} \mathbf{P}^\top$ yields the discrete Hamiltonian cycle solving the TSP instance.

Specifically, for any node reordering represented by a permutation matrix $\pi$, the GNN mapping satisfies
\begin{equation}
f_{\mathrm{GNN}}(\pi f_0,\, \pi A \pi^\top) = \pi f_{\mathrm{GNN}}(f_0, A) .
\end{equation}
As a consequence, the soft permutation $\mathbb{T}$ produced by the Gumbel--Sinkhorn operator transforms equivariantly as $\mathbb{T} \mapsto \pi \mathbb{T}$ under input permutations, and the corresponding heat map $ \mathbb{T} \mathbb{V} \mathbb{T}^\top$  becomes $\pi \mathbb{T} \mathbb{V} \mathbb{T}^\top \pi^\top$.
This ensures that all node relabelings yield equivalent solutions.

\subsection{Hamiltonian Cycle Ensemble}
To mitigate long-tail failures, \citep{min2025structure} employs an ensemble over Hamiltonian cycle structures $\mathbb{V}^k$ with $\gcd(k,n)=1$, where each $\mathbb{V}^k$ encodes a distinct structural prior over tours. While effective, this design requires training a separate permutation model for each $\mathbb{V}^k$, as the learning objective is explicitly tied to the chosen cycle structure. Consequently, the total training cost scales with the number of coprime shifts, introducing substantial overhead compared to single-model approaches and limiting scalability to larger problem sizes.

%% file: Sections/contribution.tex
\section{Our Contribution}
In this paper, we propose a single-model, unsupervised framework that encodes the Hamiltonian cycle constraint directly into the learning objective.
Our method is purely unsupervised, requiring no ground-truth solutions or reinforcement rewards. To achieve this, we introduce three elements.  First, we employ an equivariant and invariant representation that preserves permutation symmetry while removing nuisance variations such as translation, ensuring that the model operates directly on the intrinsic structure of the problem. Second, we use dropout as a source of controlled stochasticity within the network, which improves robustness without introducing iterative refinement. Third, we adopt snapshot ensembling along a single training trajectory, allowing multiple functionally distinct solutions to be obtained without additional training cost.

Taken together, these elements clarify how global structure and controlled stochasticity can be incorporated into a single-pass GNN without search.

%% file: Sections/methods.tex
\section{Our Method}
To avoid the cost of training multiple models, we rely on two complementary ensemble mechanisms that require only a single training run.
First, we employ dropout during training and retain it at inference time, performing multiple stochastic forward passes with different dropout masks. This Monte Carlo dropout procedure yields diverse solutions from a single trained model and effectively acts as an ensemble without additional training.
Second, we adopt a snapshot ensemble strategy by saving multiple checkpoints along a single training trajectory (e.g., at epochs 850, 900, 950, and 1000) and using them jointly at inference. Since all snapshots are obtained from the same optimization run, this approach introduces no extra training cost. In addition, we take a complementary architectural approach that is orthogonal to ensembling. Unlike prior work that uses raw $x_i \in \mathbb{R}^2$ coordinates as GNN inputs, we introduce a symmetry-aware feature extraction mechanism with harmonic enhancement. This design improves representational capacity while preserving equivariance, and does not incur additional training overhead.

%% file: Sections/harmonics.tex
\subsection{Equivariant Coordinate Feature Extraction}
\label{sec:equivariant-coord-features}

We introduce an \emph{equivariant coordinate feature extractor} that maps a set of planar points to per-point features while respecting the symmetries of the Euclidean group. The construction is fully deterministic and permutation-equivariant.
% The equivariant feature extraction procedure underlying these properties is illustrated in Figure~\ref{fig:equivariant-feature-extractor-compact}.

\paragraph{Problem Setting}
Let $
X = f_0 =  \{x_i\}_{i=1}^n, x_i \in \mathbb{R}^2$ denote a point set. We consider group actions of: permutations $\pi \in S_n$ acting as $(\pi \cdot X)_i = x_{\pi(i)}$;  translations $t \in \mathbb{R}^2$ acting as $(T_t \cdot X)_i = x_i + t$.

% Our goal is to construct a feature representation that is permutation-equivariant and invariant to global translations, realized by the GNN denoted in Equation~\ref{eq:fgnn}. 

We first remove translation by centering:
%\begin{align}
$c := \frac{1}{n} \sum_{i=1}^n x_i,  \quad \tilde{x}_i := x_i - c $.
%\end{align}
We then compute the empirical covariance
%\begin{align}
$\Sigma := \frac{1}{n} \sum_{i=1}^n \tilde{x}_i \tilde{x}_i^\top \in \mathbb{R}^{2 \times 2}$.
%\end{align}

Let $(\lambda_1, u_\perp)$ and $(\lambda_2, u)$ denote the eigenpairs of $\Sigma$ with $\lambda_1 \le \lambda_2$.
We fix the sign ambiguity by enforcing $u_1 \ge 0$ and apply the same sign to $u_\perp$, yielding a deterministic orthonormal frame
\begin{align}
U := [u_\perp, u] \in \mathrm{O}(2)
:= \{Q \in \mathbb{R}^{2\times 2} \mid Q^\top Q = I\}.
\end{align}

This defines a \emph{data-dependent canonical coordinate system}.

\paragraph{Intrinsic Coordinates and Harmonic Features}

Each point is projected onto the canonical frame:
\begin{equation}
a_x(i) := \tilde{x}_i^\top u,
\qquad
a_y(i) := \tilde{x}_i^\top u_\perp .
\end{equation}

We define intrinsic polar coordinates
\begin{align}
r(i) &:= \sqrt{a_x(i)^2 + a_y(i)^2 + \varepsilon}, \\
\theta(i) &:= \mathrm{atan2}\bigl(a_y(i), a_x(i)\bigr).
\end{align}

Angular information is encoded using $M$ Fourier harmonics:
\begin{equation}
\begin{aligned}
\phi_m^{\sin}(i) &:= \sin\bigl(m\,\theta(i)\bigr), \\
\phi_m^{\cos}(i) &:= \cos\bigl(m\,\theta(i)\bigr),
\end{aligned}
\qquad m = 1,\dots,M .
\end{equation}

The final per-point feature vector is
\begin{equation}
\begin{aligned}
f(i) = \Big[ \;
& r(i),\;
a_x(i),\;\\
& a_y(i), \phi_1^{\sin}(i),\dots,\phi_M^{\sin}(i),\;
\phi_1^{\cos}(i),\dots,\phi_M^{\cos}(i)
\; \Big].
\end{aligned}\label{eq:inputf}
\end{equation}
This feature vector is then fed into the GNN as the initial node features of the instance.

\paragraph{Equivariance and Invariance Properties}
We now formalize the symmetry properties of the construction in Equation~\ref{eq:fgnn}. As mentioned, $f_\mathrm{GNN}$  denotes a GNN viewed as a function that maps the initial node features to an 
$n$-dimensional output representation.

\begin{theorem}[Permutation Equivariance]
\label{thm:perm-equiv}
For any permutation $\pi \in S_n$,
\begin{align}
f_{\mathrm{GNN}}(\pi \cdot X, \pi A(X) \pi^\top)
= \pi \cdot f_{\mathrm{GNN}}(X, A(X)).
% f_\mathrm{GNN}(\pi \cdot X) = \pi \cdot f_\mathrm{GNN}(X).
\end{align}

\end{theorem}

\begin{proof}
All quantities defining the canonical frame (mean, covariance, eigenvectors) are symmetric functions of the point set and invariant under reindexing. The per-point map is applied independently using this shared frame, implying equivariance. Moreover, since the adjacency matrix $A(X)$ is constructed from pairwise
distances, it transforms equivariantly under permutations as
$A(\pi \cdot X) = \pi A(X) \pi^\top$.
\end{proof}

\begin{theorem}[Translation Invariance]
\label{thm:translation-inv}
For any translation $t \in \mathbb{R}^2$,
\begin{equation}\label{eq:tinv}
% f_\mathrm{GNN}(T_t \cdot X) = f_\mathrm{GNN}(X).
f_{\mathrm{GNN}}(T_t \cdot X, A(T_t \cdot X)) = f_{\mathrm{GNN}}(X, A(X)).
\end{equation}
\end{theorem}

\begin{proof}
Translation shifts the mean by $t$ but leaves the centered coordinates
$\tilde{x}_i$ unchanged. Since pairwise distances are translation invariant,
the adjacency matrix $A = \exp(-\mathbf{D}/s)$ constructed from $X$ is also unchanged
under $T_t$. Consequently, the covariance, canonical frame, intrinsic
coordinates, and all derived features are invariant, implying Equation~\ref{eq:tinv}.
\end{proof}

%% file: Sections/dropout.tex
\subsection{Dropout Regularization}
\paragraph{Scattering-Attention Graph Neural Network}
Our model is a scattering-attention graph neural network (SCT-GNN) that combines diffusion-based representations with learnable attention-based mixing.
Given node features $X \in \mathbb{R}^{N \times d}$ and a weighted adjacency matrix $W \in \mathbb{R}^{N \times N}$, we first construct a set of graph operators capturing multi-scale structure~\citep{min2022can}.
In particular, we use (i) normalized graph convolution operators and (ii) scattering operators derived from powers of a smoothed random-walk matrix.
Each operator defines a diffusion channel that propagates information at a distinct scale.

At layer $\ell$, node representations are updated by aggregating diffusion features across channels:
\begin{align}
H^{(\ell)} = \mathrm{Mix}\Bigl(\bigl\{ \phi_k(W) \, H^{(\ell-1)} \bigr\}_{k=1}^C \Bigr),
\end{align}
where $\phi_k(\cdot)$ denotes the $k$-th diffusion or scattering operator, and $C$ is the total number of channels.
Rather than fixing the aggregation weights, we learn attention coefficients that adaptively combine these channels at each layer.
Concretely, for node $i$ and channel $k$, attention weights are computed as
\begin{align}
\alpha_{i,k}^{(\ell)} = \mathrm{softmax}_k \bigl( a^\top [h_i^{(\ell-1)} \,\Vert\, (\phi_k H^{(\ell-1)})_i ] \bigr),
\end{align}
and the mixed representation is given by
\begin{align}
\tilde{h}_i^{(\ell)} = \sum_{k=1}^C \alpha_{i,k}^{(\ell)} \, (\phi_k H^{(\ell-1)})_i.
\end{align}
The result is passed through a feed-forward transformation with residual connections, yielding the updated state $h_i^{(\ell)}$.

Dropout is applied within the attention weights $\alpha_{i,k}^{(\ell)}$ and the subsequent feed-forward projections, injecting stochasticity into the feature mixing process.
Importantly, this stochasticity affects only intermediate representations and does not modify the final output layer or the discrete projection used to construct tours.
As a result, the model remains a single-pass heuristic, while optionally supporting stochastic inference through internal representation-level perturbations.

It should be noted that all diffusion operators and attention mechanisms are permutation-equivariant by construction, ensuring that the model respects the symmetry of the underlying graph.

\paragraph{Stochastic Inference via MC Dropout.}
At inference time, we optionally introduce stochasticity through Monte Carlo (MC) dropout to sample diverse solutions from a single trained model.
When MC dropout is disabled, the inference procedure is deterministic: for a fixed input instance, the model produces a unique output and hence a single tour.
Enabling MC dropout injects controlled randomness into the network computation, causing repeated forward passes on the same instance to yield different outputs.
While the resulting perturbations are small in the continuous output space, they may lead to qualitatively different tours after the discrete projection to a permutation.
MC dropout does not involve search, instead, it provides a lightweight means of generating solution diversity, allowing a trade-off between computational cost and solution diversity.

%% file: Sections/snapshot.tex
\subsection{Snapshot Ensembles}
Snapshot ensembling exploits the dynamics of a single training trajectory to obtain multiple diverse models without repeated training~\citep{huang2017snapshot}.
The key insight is that modern neural network optimization traverses a sequence of distinct regions in parameter space before convergence.
Intermediate checkpoints along this trajectory can therefore serve as effective ensemble members, even though they are produced by a single optimization run.

In our setting, training GNNs entails navigating a highly non-convex loss landscape. Consequently, model parameters obtained at different epochs, including late in training, can correspond to functionally distinct solutions with comparable loss values but differing empirical performance. Ensembling these checkpoints reduces variance and mitigates failure modes associated with any single model instance. Importantly, this diversity emerges naturally from the optimization dynamics and requires neither explicit architectural modifications nor multiple random initializations.

In our framework, we adopt snapshot ensembling by selecting a small set of checkpoints from a single training run.
Concretely, for a model trained over $t_\text{max}$ epochs, we select checkpoints at epochs
\[
\{t_1, t_2, \dots, t_m\}, \quad \text{e.g., } \{850, 900, 950, 1000\},
\]
or more generally by sampling snapshots at fixed intervals (e.g., every 50 epochs) during the later stages of training.
Each snapshot is then used independently at inference time, and the final solution is selected or aggregated across snapshots.

In practice, we form a simple ensemble by performing multiple stochastic forward passes and selecting, for each instance, the best tour among the sampled outputs, effectively realizing an ensemble over implicit models without additional training or architectural changes. Our stochasticity arises entirely within the network forward pass and does not involve iterative refinement or neighborhood exploration.
Unlike approaches used in ~\citep{min2025structure} that require training multiple models with different structural priors, our methods shift diversity generation entirely into the training dynamics.
This allows us to recover the robustness benefits of ensembling while avoiding the substantial computational overhead associated with repeated end-to-end training.

%% file: Sections/Training.tex
\section{Training and Inference}
We train our scattering-attention GNN using distributed data parallelism on multi-node GPU clusters.
For TSP100 and TSP200, we launch runs on $2$ nodes with $4$ GPUs per node (8 GPUs total), using per-GPU batch size $256$ and the Adam optimizer with weight decay $2.5\times 10^{-5}$, dropout $0.1$, and a learning rate $8\times 10^{-3}$. We use 16-layer and 20-layer GNNs with  temperature $\tau=3$ and hidden dimension of 256 for TSP 100 and TSP 200, respectively.
We train for up to $1000$ epochs for TSP 100 and 200 and 500, with a scheduler and a $15$-epoch warmup, enable adaptive gradient clipping, and apply early stopping with patience $100$; checkpoints are saved every $5$ epochs with automatic resume enabled.
For TSP 500, we scale to $4$ nodes with $4$ GPUs per node (16 GPUs total) and train a larger model (hidden dimension $384$, $48$ layers) with dropout $0.5$, learning rate $2\times 10^{-3}$, weight decay $10^{-4}$, and temperature $\tau=3.5$, using a smaller per-GPU batch size of $50$ due to memory constraints.
Across all settings, training and validation instances are sampled from uniform Euclidean TSP distributions, and we fix the distance scale to $5.0$ to match the inference-time construction of the graph adjacency. 

\begin{figure*}[tbp]
    \centering
    \begin{subfigure}{0.32\linewidth}
        \centering
        \includegraphics[width=\linewidth]{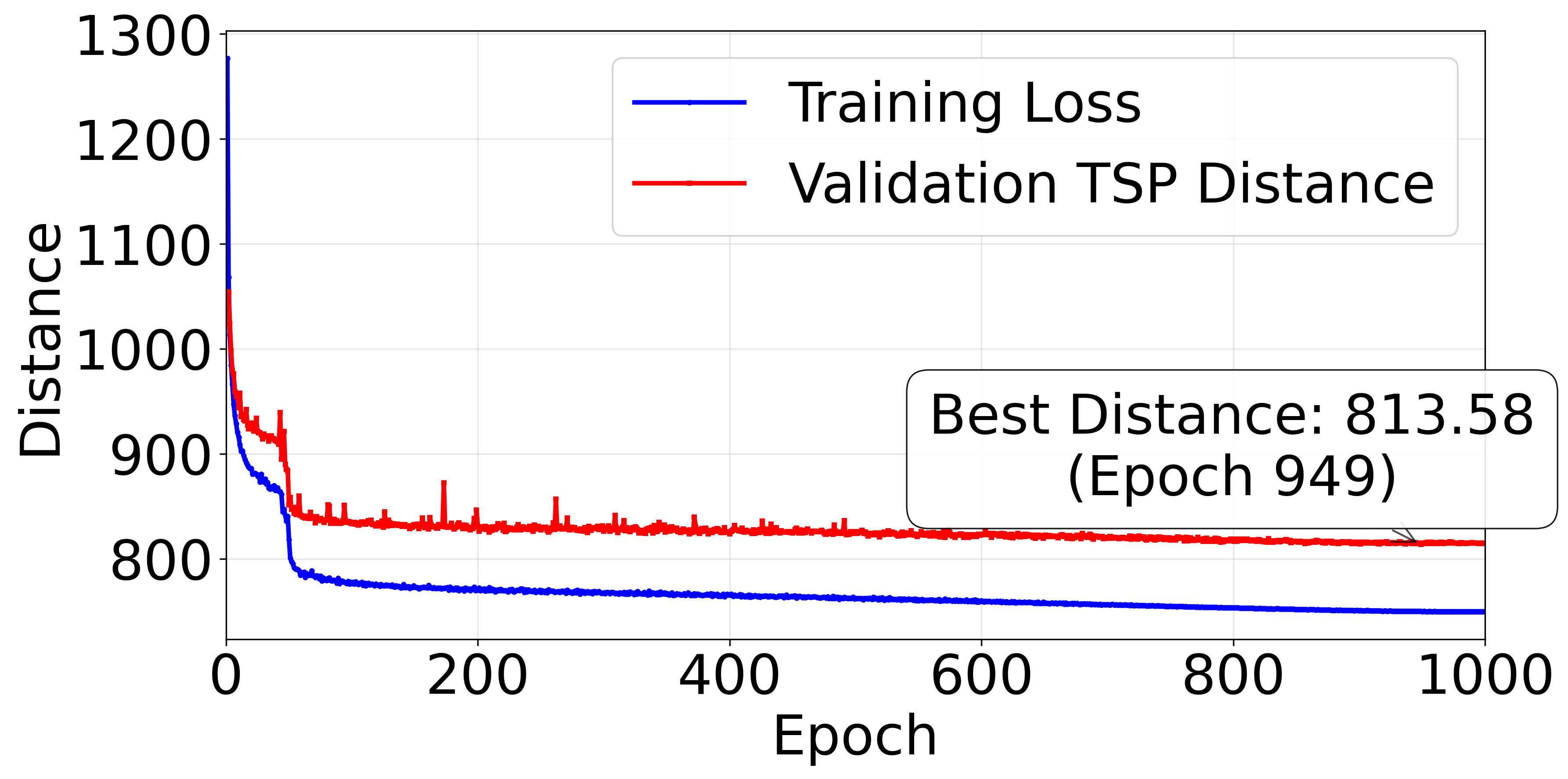}
        \caption{TSP 100}
        \label{fig:TSP100traininghis}
    \end{subfigure}
    \hfill
    \begin{subfigure}{0.32\linewidth}
        \centering
        \includegraphics[width=\linewidth]{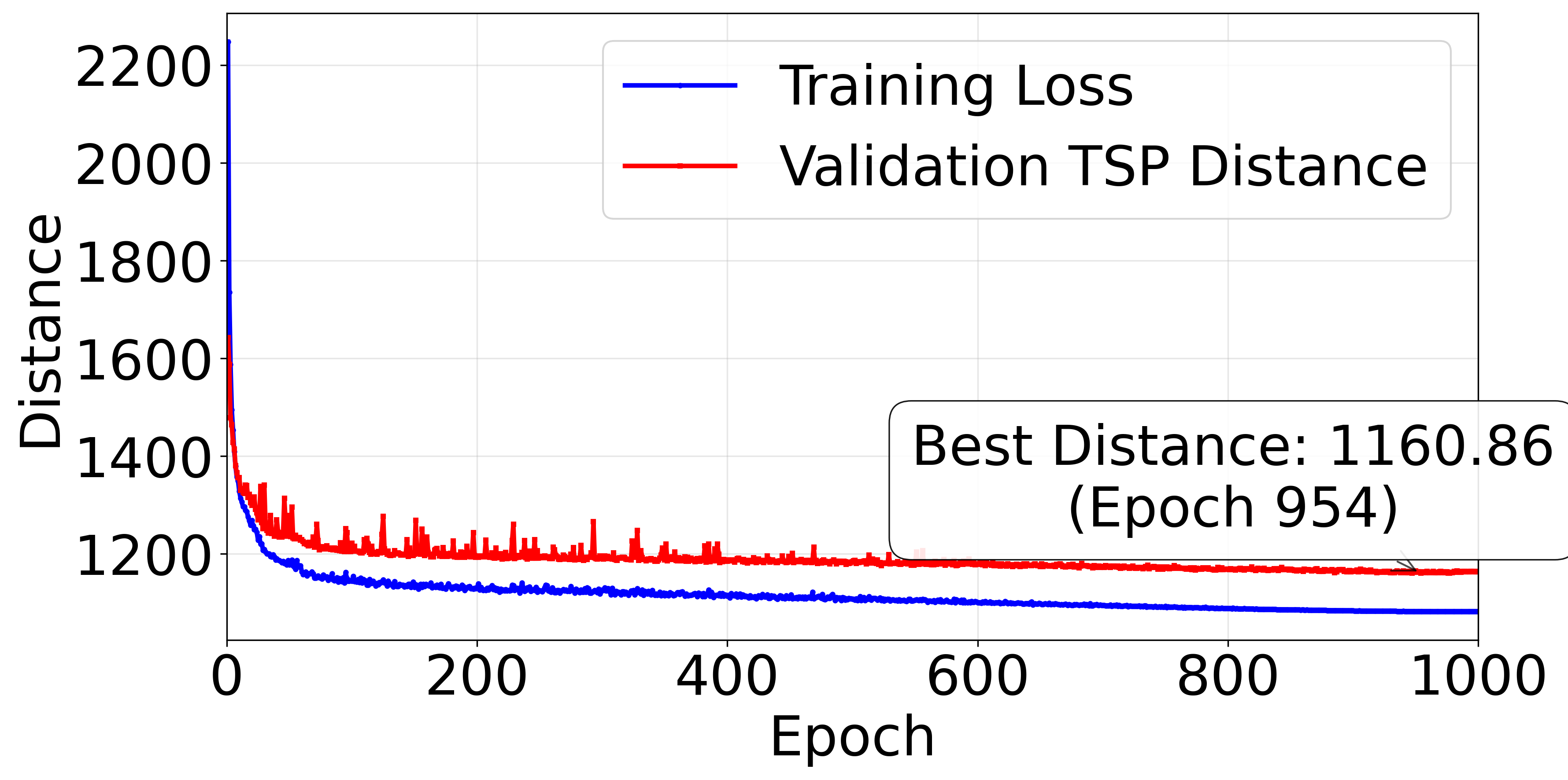}
        \caption{TSP 200}
        \label{fig:TSP200traininghis}
    \end{subfigure}
    \hfill
    \begin{subfigure}{0.32\linewidth}
        \centering
        \includegraphics[width=\linewidth]{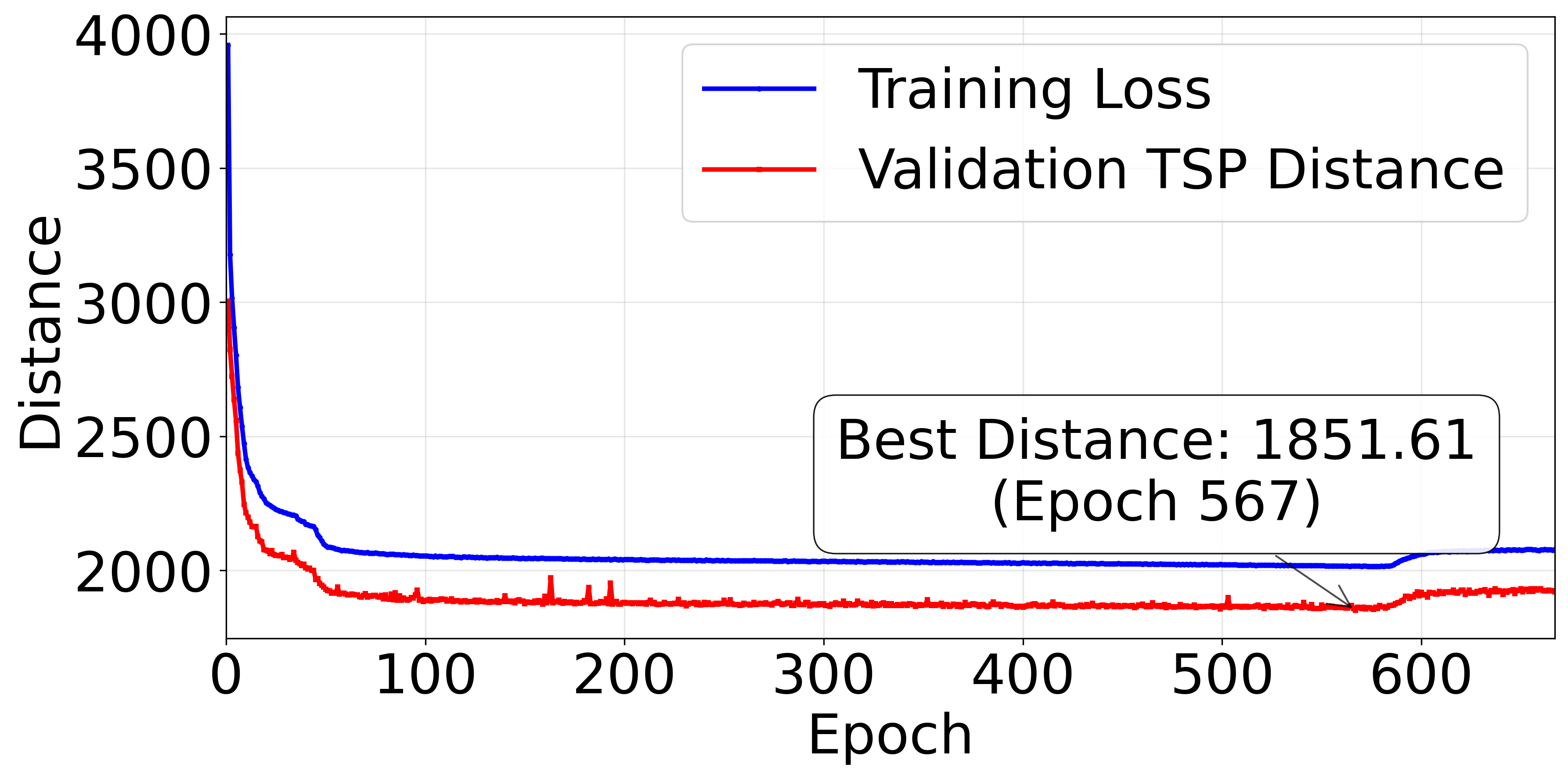}
        \caption{TSP 500}
        \label{fig:TSP500traininghis}
    \end{subfigure}

    \caption{
    Training history on different sizes.
    }
    \label{fig:TSPtraining}
\end{figure*}

The training datasets consist of $1.5$, $2.0$, and $5.0$ million uniformly sampled instances for TSP 100, TSP 200, and TSP 500, respectively.
For each setting, we use $1{,}000$ instances for validation and $1{,}000$ for testing. All experiments are conducted on a compute cluster equipped with Intel Xeon Gold 6154 CPUs and NVIDIA A100 GPUs.

%% file: Sections/results.tex
\section{Results}
Figure~\ref{fig:TSPtraining} shows the training history across different problem sizes. In all cases, the training objective decreases rapidly during the early stages of optimization and stabilizes as training progresses, indicating consistent convergence behavior. The validation curves closely track the training loss, with no signs of instability or divergence, suggesting that the learned representations generalize well across instances.
% \subsection{Length Distribution}
\begin{figure*}[ht]
    \centering
    \begin{subfigure}{0.32\linewidth}
        \centering
        \includegraphics[width=\linewidth]{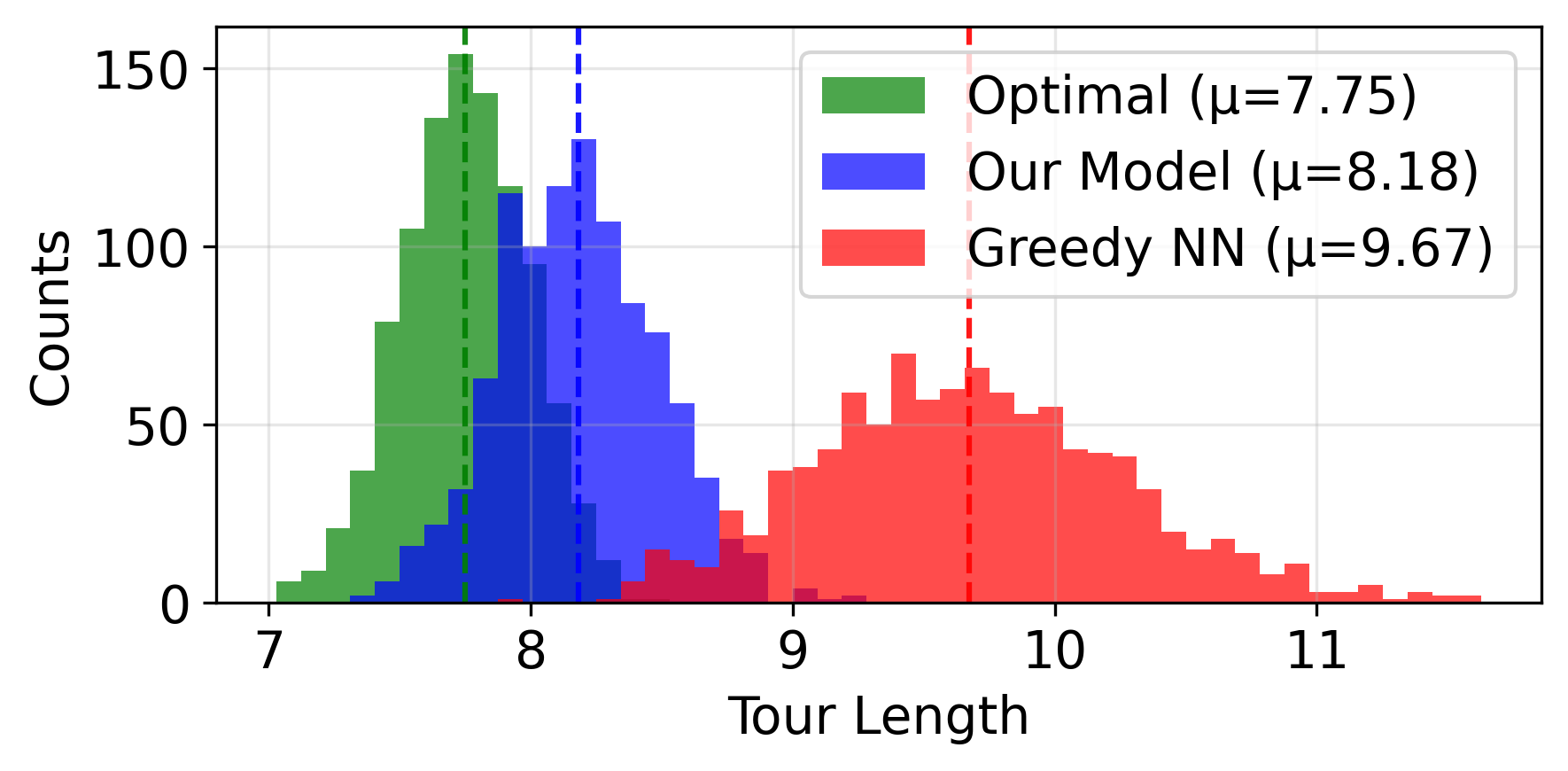}
        \caption{TSP 100}
        \label{fig:tsp100statis}
    \end{subfigure}
    \hfill
    \begin{subfigure}{0.32\linewidth}
        \centering
        \includegraphics[width=\linewidth]{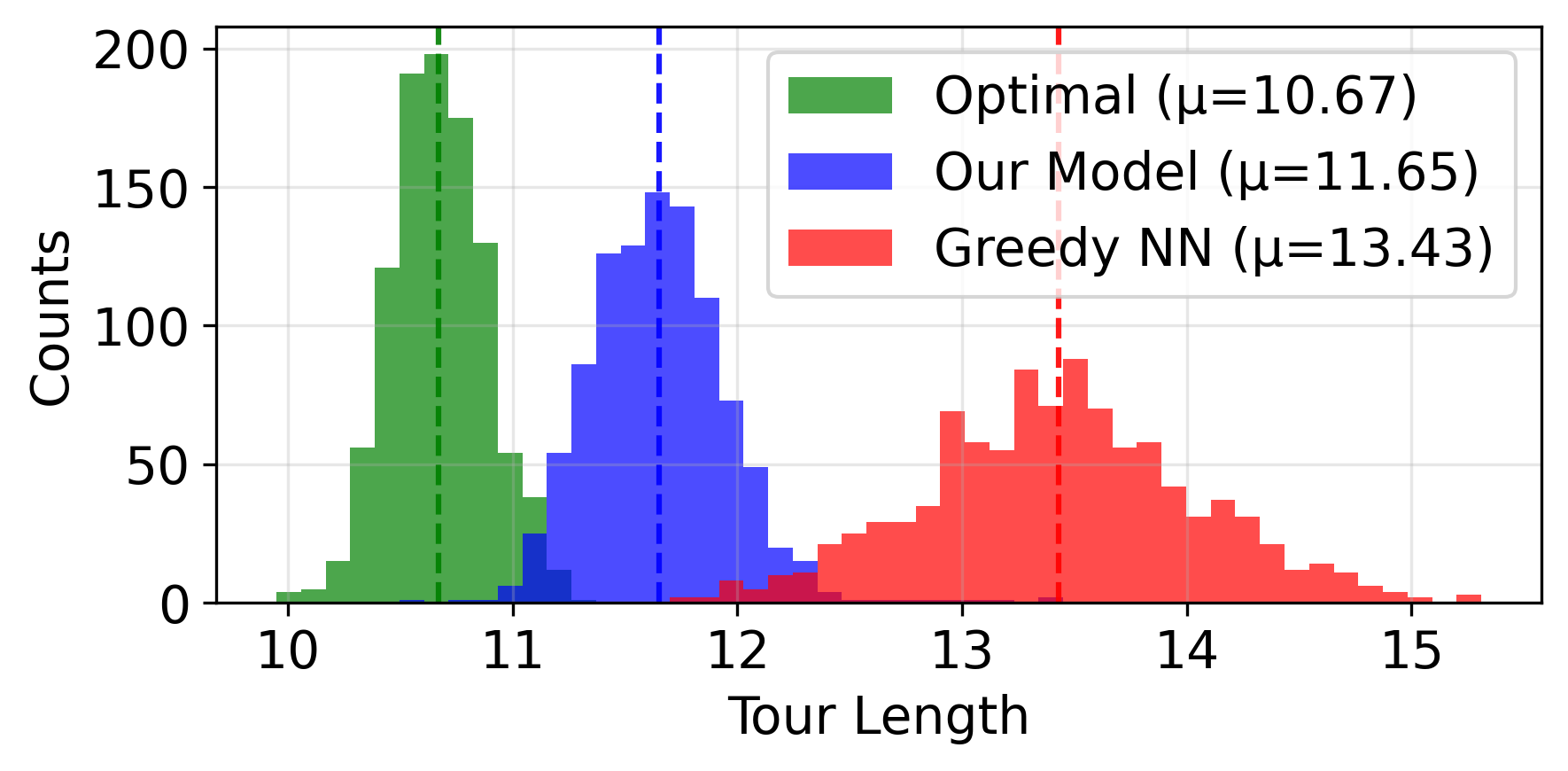}
        \caption{TSP 200}
        \label{fig:tsp200statis}
    \end{subfigure}
    \hfill
    \begin{subfigure}{0.32\linewidth}
        \centering
        \includegraphics[width=\linewidth]{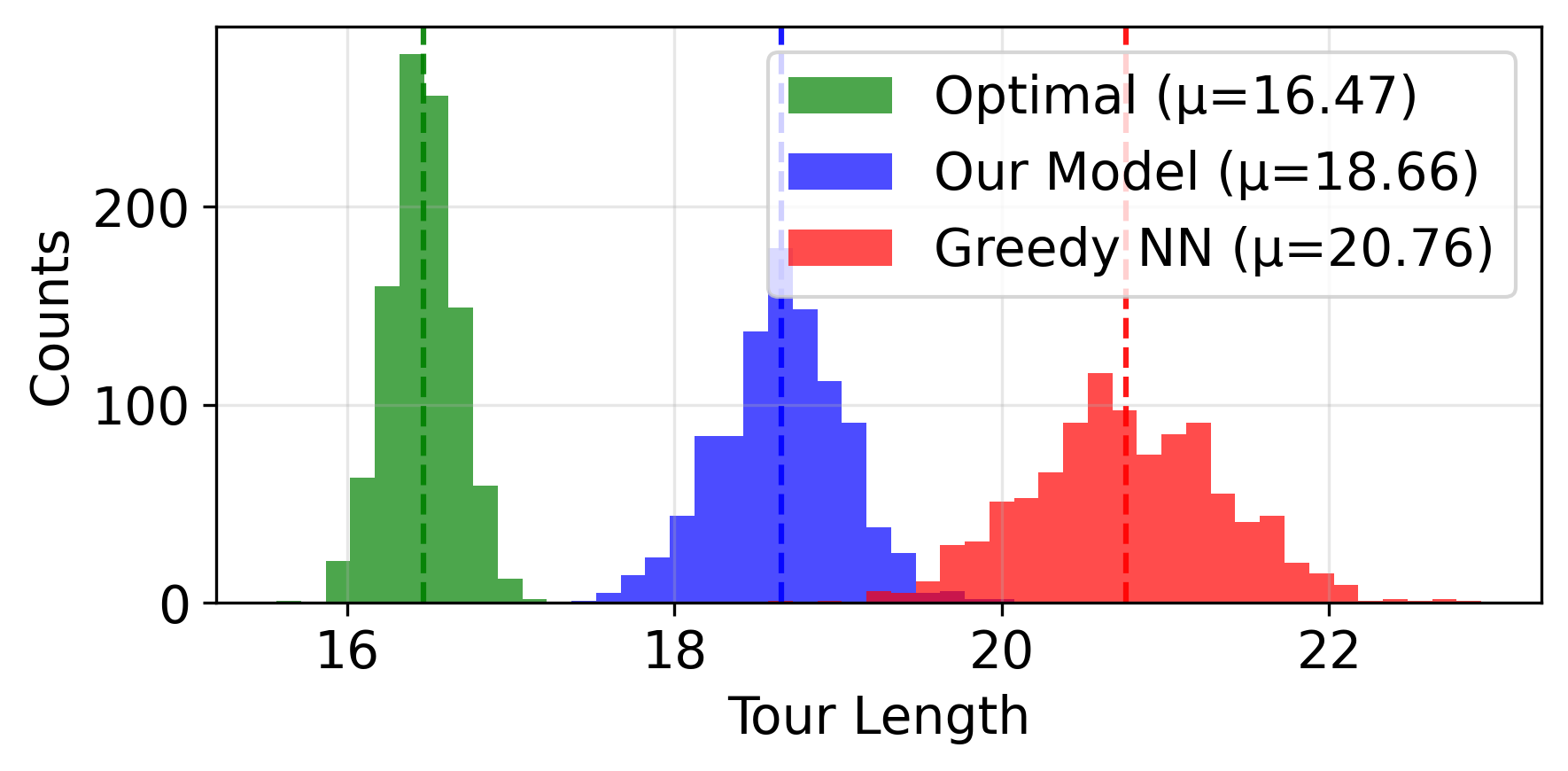}
        \caption{TSP 500}
        \label{fig:tsp500statis}
    \end{subfigure}

    \caption{
    Comparison showing our model vs. greedy nearest neighbor baseline.
    }
    \label{fig:lengthdis}
\end{figure*}
Figures~\ref{fig:lengthdis} show the tour-length distributions on the test set for 100-, 200-, and 500-node instances, using the model that achieves the lowest validation length across all hyperparameter settings. Across all problem sizes, our method consistently outperforms the Greedy Nearest Neighbor (NN) baseline, which constructs tours by iteratively selecting the closest unvisited node. The resulting histograms show that our model produces shorter and more tightly concentrated tour-length distributions, with mean lengths of $\mu = 8.18$, $11.65$, and $18.66$, compared to Greedy NN’s $\mu = 9.67$, $13.43$, and $20.76$, respectively.

\input{Tables/tablemcdropoutensemble}
Table~\ref{table:mc_dropout_ensemble} illustrates the behavior of our MC-dropout ensemble on Euclidean TSP instances with \(n \in \{100, 200, 500\}\). Each column corresponds to a single inference run of the same trained model evaluated on the test set. The deterministic run (without MC dropout) consistently achieves both the lowest average tour length and the highest win rate across all problem sizes (e.g., 18.4\% on TSP100), highlighting the stability of a single-pass heuristic. In contrast, individual stochastic runs yield slightly higher average tour lengths, but collectively generate a diverse set of solutions, reflecting the controlled variability introduced by MC dropout.

\input{Tables/tablesnapshotensemble}

\input{Tables/table1}

Table~\ref{table:snapshot_ensemble} illustrates the behavior of our snapshot ensemble across problem sizes. Each column corresponds to a single inference run using the same model architecture, but taken from different training epochs. For TSP 100 and TSP 200, snapshots are drawn from training epochs \(\{850, 900, 950, 1000\}\), while for TSP 500 they are taken from \(\{400, 450, 500,550\}\).

The deterministic model (equivalently, the best-validation snapshot) is consistently among the strongest single predictors, achieving the lowest average tour length for \(n=100\), tying for the lowest average tour length for \(n=200\), and remaining the strongest at \(n=500\). However, the snapshot that maximizes win rate is not always the best-validation one. In particular, for \(n=200\), Snapshot-2 attains the highest win rate (22.1\%), exceeding that of the deterministic model (21.7\%). This observation highlights that different snapshots exhibit distinct trade-offs between average solution quality and instance-wise dominance.

Although individual snapshots generally produce slightly higher average tour lengths than the deterministic model, they generate a diverse set of candidate tours. This diversity is analogous in spirit to MC-dropout ensembles, yet does not rely on injecting stochasticity at inference time. By selecting the best solution among multiple deterministic snapshots, the snapshot ensemble consistently improves the overall mean tour length across all problem sizes while preserving fully deterministic inference.

Importantly, our ensemble models exploit the diversity without introducing search or iterative refinement. These results indicate that the observed diversity arises from optimization-induced variations in the learned representations rather than explicit exploration, thereby preserving the non-iterative nature of the heuristic. Table~\ref{table:table1} reports the performance of our model and the compared methods. We compare against classical heuristics implemented using widely adopted and optimized reference codebases.  Specifically, the Christofides heuristic is implemented using the \texttt{networkx} library, while the Greedy Nearest-Neighbor and 2-opt heuristics are implemented in optimized C++ with high efficiency.
Our method is evaluated using batched inference with batch size 512 on an A100 GPU, whereas classical heuristics are executed on CPU~\citep{hagberg2020networkx}.

Despite being a purely feed-forward model that constructs a complete tour in a single pass, our method reduces the optimality gap of Greedy by more than a factor of two on TSP 100 and TSP 200, and by nearly half on TSP 500. This improvement does not arise from local refinement or search, but from the model’s internalization of global TSP structure through the objective in Equation~\ref{eq:tsploss}. Compared to prior learning-based approaches, our method occupies a fundamentally different point in the quality–efficiency landscape. Reinforcement learning and supervised models often rely on autoregressive decoding, beam search, or hybrid pipelines with local improvement heuristics to achieve competitive performance, incurring substantial computational overhead at inference time.  In contrast, our GNN operates as a standalone heuristic: a single forward pass produces a complete solution, yet achieves solution quality that is competitive with, and in several cases superior to, stronger learning-based baselines.

Across all problem sizes, we observe a clear and smooth trade-off between solution quality and inference cost as different ensemble strategies are applied to the same unsupervised GNN. Even in the deterministic setting, a single forward pass already achieves strong performance, with optimality gaps of 5.55\%, 9.18\%, and 13.30\% on TSP 100, TSP 200, and TSP 500, respectively, at sub-millisecond to few-millisecond latency. demonstrating that the model internalizes substantial global structure without search or iterative refinement. Snapshot ensembling over five checkpoints and MC-dropout ensembling over ten stochastic passes further and consistently reduce the gap by exploiting complementary sources of diversity arising from optimization dynamics and controlled representation-level stochasticity, respectively, with the combined strategy achieving the strongest results overall (4.39\%, 8.25\%, and 12.11\%). Importantly, these gains are obtained without additional training runs, neighborhood exploration, or post hoc local improvement, and inference remains firmly in the millisecond regime. Compared to classical heuristics, this yields a consistently more favorable quality--efficiency trade-off: while Christofides attains gaps of 12.12\%, 12.75\%, and 13.30\% on TSP 100, TSP 200, and TSP 500, respectively, our feed-forward GNN substantially outperforms it on smaller instances and matches it on TSP 500, at orders-of-magnitude lower inference cost on GPUs. More aggressive heuristics such as 2-opt can further reduce the gap, but only at the expense of significantly higher runtime, particularly as problem size grows. Overall, these results suggest that learned, structure-aware heuristics can rival or surpass classical hand-designed methods, while naturally inducing a continuous and controllable trade-off between solution quality and inference cost.

%% file: Tables/tablemcdropoutensemble.tex
\begin{table}[!htb]
\small
\setlength{\tabcolsep}{2pt}   
\centering
\caption{Performance of a size-10 MC-dropout ensemble evaluated on Euclidean TSP instances with \(n \in \{100, 200, 500\}\). The win rate reports the percentage of test instances for which a given run produces the shortest tour among all runs in the ensemble for the same instance.}\label{table:mc_dropout_ensemble}
\begin{tabular}{llcccccccccc|c}
\toprule
 {Size} & {Metric}
& {Deterministic}
& {MC-1}
& {MC-2}
& {MC-3}
& {MC-4}
& {MC-5}
& {MC-6}
& {MC-7}
& {MC-8}
& {MC-9}& {Ensemble} \\
\hline
\multirow{2}{*}{{100}}
% \hline
& Mean length
& 8.18
& 8.23
& 8.23
& 8.22
& 8.23
& 8.23
& 8.23
& 8.23
& 8.23
& 8.23 
& \multirow{2}{*}{{8.11}} \\

& Wins (\%)
& 18.4
& 9.9
& 11.0
& 9.1
& 8.0
& 8.7
& 9.3
& 9.3
& 7.9
& 8.4 
& \\
\hline
\multirow{2}{*}{{200}}
& Mean length
& 11.65
& 11.77
& 11.77
& 11.76
& 11.77
& 11.77
& 11.77
& 11.77
& 11.77
& 11.78
& \multirow{2}{*}{{11.57}}\\

& Wins (\%)
& 27.4
& 8.4
& 7.4
& 10.3
& 8.1
& 8.3
& 8.4
& 8.2
& 6.3
& 7.2 
& \\\hline
\multirow{2}{*}{{500}}
& Mean length
& 18.66
& 18.70
& 18.71
& 18.70
& 18.71
& 18.71
& 18.71
& 18.70
& 18.71
& 18.70
& \multirow{2}{*}{{18.52}}\\

& Wins (\%)
& 14.3
& 9.3
& 10.2
& 9.2
& 9.6
& 9.1
& 9.1
& 10.1
& 10.6
& 8.5 
&\\
\bottomrule

\end{tabular}
\end{table}

%% file: Tables/tablesnapshotensemble.tex
\begin{table}[!htb]
\small
\setlength{\tabcolsep}{2pt}   
\centering
\caption{Performance of a size-5 snapshot ensemble evaluated on Euclidean TSP instances with \(n \in \{100, 200, 500\}\).}\label{table:snapshot_ensemble}
\begin{tabular}{llccccc|c}
\toprule
 {Size} & {Metric}
& {Deterministic}
& {Snapshot-1}
& {Snapshot-2}
& {Snapshot-3}
& {Snapshot-4}
& {Ensemble} \\
\hline
\multirow{2}{*}{{100}}
& Mean length
& 8.18
& 8.20
& 8.19
& 8.18
& 8.18
& \multirow{2}{*}{{8.12}} \\

& Wins (\%)
& 27.5
& 20.0
& 21.3
& 10.6
& 20.6
& \\
\hline
\multirow{2}{*}{{200}}
& Mean length
& 11.65
& 11.68
& 11.66
& 11.66
& 11.65
& \multirow{2}{*}{{11.57}}\\

& Wins (\%)
& 21.7
& 20.3
& 22.1
& 20.6
& 15.3
& \\
\hline
\multirow{2}{*}{{500}}
& Mean length
& 18.66
& 18.72
& 18.72
& 18.70
& 18.68
& \multirow{2}{*}{{18.51}}\\

& Wins (\%)
& 29.8
& 15.4
& 16.3
& 17.1
& 21.4
& \\
\bottomrule
\end{tabular}
\end{table}

%% file: Tables/table1.tex
\begin{table}[!htb]
\small
\setlength{\tabcolsep}{1pt}   
\renewcommand{\arraystretch}{1.15} 
% --------------------------------------------------------
\centering
 \caption{Performance comparison on Euclidean TSP instances with $n \in \{100,200,500\}$, where all instances are drawn from a uniform distribution in the unit square.
We report average tour length, optimality gap relative to Concorde, and average runtime.
All learning-based baselines are taken from~\citep{fu2021generalize}, while classical methods are evaluated using widely adopted reference implementations.
Our method is a \emph{pure heuristic}: it constructs a complete tour in a single forward pass, without search. 
}\label{table:table1}
    \begin{tabular}{l|l|ccc|ccc|ccc} \toprule
       \multirow{2}{*}{Method} & \multirow{2}{*}{Type}  & \multicolumn{3}{c}{TSP100} & \multicolumn{3}{c}{TSP200} & \multicolumn{3}{c}{TSP500}\\
         &  & Length & Gap & Time & Length & Gap & Time & Length & Gap & Time \\
        \hline
                Concorde & Solver &  {7.75} & {0.00\%} & {0.21s} & {10.67} & {0.00\%} & {1.18s} & {16.47} & {0.00\%} & {14.25s} \\

        \hline
        GAT \citep{deudon2018learning}   & RL, S &  {8.83} &  {13.86\%} &   {0.29s} &  {13.17} &  {22.91\%} &   {2.27s} &  {28.63} &  {73.03\%} &   {9.46s}  \\

        GAT~\citep{kool2018attention}  & RL, S& {7.97} & {2.74\%} & {0.44s}  &  {11.45} &  {6.82\%} &   {2.10s} & {22.64} & {36.84\%} & {7.33s} \\

        GAT~ \citep{kool2018attention}    & RL, G& {8.10} & {4.38\%} & {0.01s} & {11.61} & {8.31\%} & {0.04s} & {20.02} & {20.99\%} & {0.71s}  \\

        GAT~\citep{kool2018attention}   & RL, BS& {7.95} & {2.48\%} & {0.60s}  & {11.38} & {6.14\%} & {2.70s} & {19.53} & {18.02\%} & {10.31s}  \\

        GCN~\citep{joshi2019efficient} & SL, G & {7.88} &  {1.48\%} &   {0.19s} & {17.01} & {58.73\%} &  {0.46s} & {29.72} & {79.61\%} & {3.13s}   \\

        GCN~\citep{joshi2019efficient} & SL, BS &{7.88} & {1.48\%} & {0.19s} &  {16.19} &  {51.02\%} &   {2.17s} &  {30.37} &  {83.55\%} &   {17.82s}  \\\hline

        \tabincell{l}{Christofides\\ \citep{christofides1976worst}} & \tabincell{l}{Heuristic\\ (networkx)}  &{8.69} & {12.12\%} & {0.11s} & {12.03} & {12.75\%} & {0.58s} & {18.66} & {13.30\%} & {6.00s} \\

        2-opt & \tabincell{l}{Heuristic\\ (C++)}  &{8.69} & {12.13\%} & {1.29s} & {11.92} & {11.72\%} & {5.99s} & {18.33} & {11.29\%} & {44.65s} \\
        \tabincell{l}{Greedy\\ Nearest-Neighbor} & \tabincell{l}{Heuristic\\ (C++)} &{9.67} & {24.77\%} & {0.04ms} & {13.43} & {25.87\%} & {0.11ms} & {20.76} & {26.05\%} & {0.63ms} \\
        \hline

        \tabincell{l}{Ours (batch size = 512)\\no MC dropout} & {GNN, UL}  &{8.18} & {5.55\%} & {0.34ms} & {11.65} & {9.18\%} & {0.96ms} & {18.66} & {13.30\%} & {5.84 ms} \\ \hline
          \tabincell{l}{Ours (batch size = 512)\\snapshot  ensembling: 5 } & {GNN, UL}  &{8.12} & {4.77\%} & {1.69ms} & {11.57} & {8.43\%} & {4.75ms} & {18.51} & {12.39\%} & {29.19ms} \\  \hline
        \tabincell{l}{Ours (batch size = 512)\\dropout ensembling: 10 } & {GNN, UL}  &{8.11} & {4.65\%} & {3.37ms} & {11.57} & {8.43\%} & {9.48ms} & {18.52} & {12.45\%} & {59.34ms} \\ \hline
        \tabincell{l}{Ours (batch size = 512)\\ dropout + snapshot } & {GNN, UL}  &{8.09} & {4.39\%} & {4.70ms} & {11.55} & {8.25\%} & {13.28ms} & {18.47} & {12.11\%} & {82.63ms} \\

        \bottomrule
    \end{tabular}
\end{table}

%% file: Sections/conclusion.tex
\section{Conclusion}
We propose an unsupervised, non-autoregressive graph neural network heuristic that constructs complete TSP solutions in a single forward pass by directly encoding global combinatorial constraints into the learning objective, without reliance on search. Our performance gains arise from three elements: symmetry-aware equivariant representations that preserve permutation structure, controlled stochasticity via dropout to improve robustness, and snapshot ensembling along a single training trajectory to obtain diverse solutions at no additional training cost. Together, these elements position graph neural networks as effective heuristics for combinatorial optimization and suggest a path toward broader applicability through more expressive symmetry-aware architectures and extensions to related problems.

%% file: Sections/acknowledgement.tex
\section{Acknowledgement}
This project is partially supported by the Eric and Wendy Schmidt AI
in Science Postdoctoral Fellowship, a Schmidt Futures program; the National Science Foundation
(NSF) and the  National Institute of Food and Agriculture (NIFA); the Air
Force Office of Scientific Research (AFOSR);  the Department of Energy;  and the Toyota Research Institute (TRI).